\documentclass{article}

\usepackage[english]{babel}

\usepackage[letterpaper,top=2cm,bottom=2cm,left=3cm,right=3cm,marginparwidth=1.75cm]{geometry}

\usepackage{amsmath}
\usepackage{amsthm}
\usepackage{graphicx}
\usepackage[colorlinks=true, allcolors=blue]{hyperref}
\usepackage{natbib}
\usepackage{booktabs}
\usepackage{amssymb}
\usepackage{mathtools}
\usepackage[ruled,vlined]{algorithm2e}
\usepackage{dsfont}
\usepackage{tabularx}
\usepackage{pifont}
\usepackage{multirow}
\usepackage{makecell}
\usepackage[bottom]{footmisc}

\newcommand{\cmark}{\ding{51}}
\newcommand{\xmark}{\ding{55}}

\SetKwInput{KwInput}{Input}
\SetKwInput{KwOutput}{Output}

\newcolumntype{L}[1]{>{\raggedright\arraybackslash}p{#1}}
\newcommand{\arrowtag}[1]{\hfill(\(#1\))}

\theoremstyle{plain}
\newtheorem{theorem}{Theorem}

\newtheorem{corollary}[theorem]{Corollary}
\theoremstyle{definition}

\theoremstyle{remark}

\title{Distribution Matching Evolutionary Algorithms for\\ Rare Event Sampling}
\author{
  Yonatan Gideoni \quad\quad Yarin Gal \\[0.5em]
  University of Oxford \\
  \texttt{yg@robots.ox.ac.uk}
}
\date{}

\begin{document}
\maketitle

\begin{abstract}
A novel discovery is one which is both useful and surprising: a generative model’s output is a useful discovery if it has a low probability of being generated (it's surprising) and a high reward (it's useful). Global optimization can directly increase the probability of sampling high rewards but typically requires updating model weights. Such gradient based optimization is expensive and bars using capable closed-source models. Instead, modern search methods for discovery sacrifice the global target, and use evolutionary algorithms with local reward maximizing objectives, permitting the search to focus only on high probability samples. In this paper, we interpret various evolutionary algorithms as approximate Markov Chain Monte Carlo, an optimization-free method to sample from complex distributions. This interpretation allows developing Distribution Matching Evolutionary Algorithms (DME), a class of search methods which sample from a global target distribution without updating weights. Empirically, DME has a higher sample efficiency than existing methods on problems requiring many samples to find a solution.\footnote{For code see \href{https://github.com/YonatanGideoni/dme}{https://github.com/YonatanGideoni/dme}.}
\end{abstract}

\vspace{-12pt}
\section{Introduction}
\vspace{-4pt}
This paper is concerned with the problem of how to find novel discoveries using generative models. Specifically, what does it mean for a discovery to be novel? Intuitively, a novel discovery is one which has never been seen before and is in some sense useful. Whether a discovery has been seen can be judged based on prior work, while its utility can be measured using some verifier or reward.\footnote{There are very interesting, harder cases where how to measure utility is unclear -- although highly relevant, we do not discuss those here.} When searching for a discovery using a generative model, if it is novel then the model will likely assign it a low generation probability, due to not seeing the discovery in its training data.

A common way to leverage generative models to find discoveries is using evolutionary algorithms \citep{novikov2025alphaevolve}, see Algorithm \ref{alg:ea} for an example. Evolutionary algorithms search for high reward samples by mutating existing candidates and selecting a subset based on their reward. Generative models are useful for mutating samples as they tend to yield semantically meaningful mutations, e.g. changing a program in a way so it still compiles. A common selection operator is to keep only the highest-reward candidates, which maximizes the local reward.

\begin{algorithm}[!b]
\label{alg:ea}
\caption{Example Evolutionary Algorithm}

\KwInput{An initial parent $x_0$, a mutation operator $q(x'|x)$, a selection operator $S$, a fitness function $F(x)$, population size $N$, number of generations $n$}

\For{$i=0...n$}{
    Generate a new population of offspring $x'_1,x'_2,...,x'_N\sim q(\cdot|x_i)$
    
    Evaluate offspring fitnesses $F(x'_1),F(x'_2),...,F(x'_N)$

    Select one of the offsprings as the next parent $x_{i+1}$ using $S$
}

Return the fittest offspring

\end{algorithm}

Local reward maximization methods rarely sample low probability outputs due to getting stuck in local modes of the distribution. A better alternative to local reward maximization would be searching for not \textit{a} high reward solution, but for \textit{all} high reward solutions \citep{bengio2021flow}. Concretely, instead of directly maximizing the reward, a more suitable goal would be to sample from a target distribution that focuses its probability mass on high reward regions. A mode collapsed search would sample solely from a single high reward region and thus be penalized. We refer to algorithms that sample from a target distribution as \textit{distribution matching} algorithms.

Generative models are typically trained for distribution matching using variational inference, where their weights are updated to minimize a KL divergence between the generated and the target distributions \citep{kingma2013auto,ho2020denoising,malkin2022gflownets,eijkelboom2024variational}. For discovery problems, training a generative model has two main downsides. First, it bars using capable fixed weights closed-source models, and second, training is much more expensive and time consuming than sampling. Due to these two reasons, most discovery methods are gradient-free evolutionary algorithms.

\looseness=-1 A wide class of alternative, optimization-free methods for sampling from a target distribution are based on Markov Chain Monte Carlo (MCMC). Although widely used in statistics \citep{brooks2011handbook}, MCMC is less prevalent in modern generative modelling due to requiring many steps to generate a single sample \citep{goodfellow2014generative,bengio2021flow}. For discovery however, methods typically anyway require many steps and once the discovery is made there is no need to generate additional samples.

Building on this insight, in this article we develop a class of inference only distribution matching algorithms. First, we demonstrate how some evolutionary algorithms can be interpreted as implementing approximate MCMC. We then use this interpretation to derive Distribution Matching Evolutionary Algorithms (DMEs), inference only search methods which sample from a target distribution. Empirically, DMEs are found to be more sample efficient than many existing methods on discovery problems which require many samples to find a solution.

\vspace{-2pt}
\section{Preliminaries}
\label{sec:dm4disc}
This section introduces and builds intuition for some relevant methods. First, we explain how sampling can be cast as optimization, and how sampling from a given target distribution allows performing global optimization. Then we discuss a relaxation of this optimal target, yielding the target distribution of \citet{bengio2021flow}. Afterwards, we present MCMC, the statistics method used here to sample from the relaxed target distribution.

\vspace{-2pt}
\subsection{Sampling as Optimization}
Sampling from a target distribution can be used to optimize a target function, with the first and arguably most well known method for doing so being simulated annealing \citep{kirkpatrick1983optimization}. In simulated annealing, one samples from $\pi_T(x)\propto \exp\left( -\beta f(x)\right)$, where $f(x)$ is the function to minimize and $\beta=1/T$ is an inverse temperature. As $T$ is decreased (``annealed'') to $0$, or equivalently $\beta\to\infty$, sampling from $\pi_T$ yields $x$s that minimize $f(x)$. Simulated annealing is useful but limited, as often sampling from $\pi_T$ is in itself nontrivial.

In our setting, in place of some $f(x)$ we presume access to some bounded positive reward function, $R(x)$. To find all high reward solutions, one would ideally like to sample $x$ uniformly from the max-reward set $S\coloneqq \{x|R(x)=\max_{x'}R(x')\}$, which would mean sampling from the idealized distribution $\pi_S(x)\coloneqq \frac{\mathbf{1}\{x \in S\}}{|S|}$, where $\mathbf{1}\{x \in S\}$ is $1$ when $x\in S$ and $0$ otherwise. Sampling from $\pi_S$ would result in sampling from \textit{all} max-reward solutions in the space and never sampling a suboptimal solution. In this sense $\pi_S$ is the optimal target distribution, although it is unattainable in practice.

\vspace{-2pt}
\subsection{Relaxing the Optimal Target}
There are two main difficulties with using $\pi_S$ as a target distribution when searching for discoveries. First, a priori the maximum reward is often unknown, e.g. if searching for new mathematical bounds the tightest possible bound is unknown in advance. Second, $\pi_S(x)$ encodes no notion of similarity or progress towards a higher reward the search algorithm may use, as it gives no probability mass to all suboptimal samples (see Figure \ref{fig:dist_relax} (b)).

A straightforward relaxation of $\pi_S$ is $\pi(x)=R(x)^\beta/Z$, where $Z$ is some unknown normalizing constant. As $\beta\to\infty$ one has that $\pi(x)\to\pi_S(x)$, as all suboptimal samples will have vanishingly small probability mass relative to the max-reward set. A higher $\beta$ will thus better approximate the ideal target but yields a harder to sample from distribution, while lower $\beta$s yield easier target distributions at the cost of more mass being placed on suboptimal solutions (see Figure \ref{fig:dist_relax} (c)).

\begin{figure}[t]
    \centering
    \includegraphics[width=\linewidth]{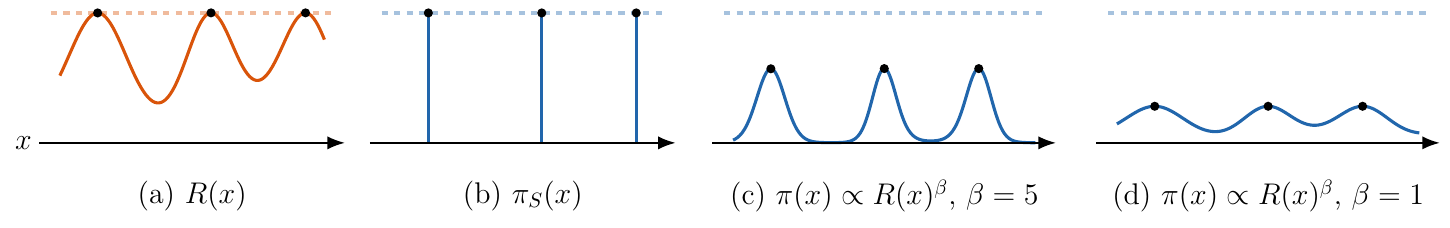}
    \vspace{-15pt}
    \caption{Given some reward function (a) sampling uniformly over all high reward solutions requires sampling from $\pi_S$ (b). $\pi_S$ is difficult to sample from due to slightly suboptimal solutions having no probability, so it is easier to sample from a relaxation (c). If relaxed too far, the target distribution will be almost uniform, leading to many wasted samples (d).}
    \label{fig:dist_relax}
    \vspace{-1pt}
\end{figure}

Historically, using $\pi(x)\propto R(x)^\beta$ as a target distribution was first introduced by \citet{bengio2021flow} for GFlowNets, although variations exist in earlier works \citep{kirkpatrick1983optimization}. GFlowNets aim to solve the same problem -- sampling diverse high reward solutions -- using methods from reinforcement learning and variational inference \citep{malkin2022gflownets}. Unlike simulated annealing, for GFlowNets $\beta$ is typically kept fixed. Both methods approximate the target distribution by updating a proposal distribution's parameters; next, we introduce a method to generate samples without needing to update the proposal distribution.

\vspace{-2pt}
\subsection{Markov Chain Monte Carlo}
To introduce MCMC it is helpful to first understand a simpler sampling algorithm, rejection sampling. Rejection sampling is a method to sample from some target distribution $\pi(x)$ given a proposal distribution $q(x)$. $q$ is needed as typically $\pi$ can be evaluated at a given point but cannot be sampled directly. Rejection sampling samples from $\pi$ by first sampling $x\sim q(\cdot)$ and accepting the sample with probability $\min\left(1,\frac{\pi(x)}{Mq(x)}\right)$, where the $\frac{\pi(x)}{q(x)}$ term is an importance sampling correction and $M$ is a constant which is sufficiently large to ensure that $Mq(x)\geq \pi(x)$. The sample is necessarily accepted if $\pi(x)\geq Mq(x)$ and is only sometimes accepted if $\pi(x)<Mq(x)$. Although simple, rejection sampling is very sample inefficient, as in high dimensional spaces there may be many regions where the target's density is much less than the proposal's. In these cases $\pi(x)\ll Mq(x)$, making the acceptance probability very small.

MCMC can be thought of as a sequential, local form of rejection sampling. MCMC is more robust to the choice of the proposal distribution than rejection sampling due to taking local steps, at the cost of needing multiple steps to generate a single independent sample from the target. The most general form of MCMC is Metropolis-Hastings MCMC \citep{metropolis1953equation,hastings1970monte}, where instead of accepting or rejecting a proposed point a change to an existing point is accepted or rejected.\footnote{Unless mentioned otherwise, when discussing MCMC we specifically refer to Metropolis-Hastings MCMC.} In MCMC the proposal distribution is replaced by a transition kernel, $q(x'|x)$, and a proposed $x'$ is accepted with probability $\min\left(1,\frac{\pi(x')q(x|x')}{\pi(x)q(x'|x)}\right)$. The $q(x'|x)$ terms arise as a correction factor, so the sampling will be unbiased even if it is likelier to transition from $x$ to $x'$ than from $x'$ to $x$. If the transition kernel is symmetric, where $q(x'|x)=q(x|x')$, the acceptance probability simplifies to $\min\left(1,\frac{\pi(x')}{\pi(x)}\right)$. Note that now the transition kernel drops out, so a proposal $x'$ is always accepted if it has a higher target probability than the previous point $x$, and is only sometimes accepted if its target probability is lower. Algorithm \ref{alg:mcmc} shows the full MCMC loop.

\begin{algorithm}
\label{alg:mcmc}
\caption{Metropolis-Hastings MCMC}

\KwInput{A starting point $x_0$, a (potentially unnormalized) target distribution $\pi(x)$, a transition kernel $q(x'|x)$, number of steps $n$}

\For{$i=0...n$}{
    Sample $x'\sim q(\cdot|x_i)$
    
    Calculate acceptance probability $a=\min\left(1,\frac{\pi(x')q(x_i|x')}{\pi(x_i)q(x'|x_i)}\right)$
    
    With probability $a$, set $x_{i+1}\leftarrow x'$, otherwise $x_{i+1}\leftarrow x_i$
}

Return all $x_i$; if $n$ is sufficiently large, they form a (correlated) set of samples from $\pi$

\end{algorithm}

\vspace{-2pt}
\section{Evolutionary Algorithms as Approximate Markov Chain\\Monte Carlo}
\vspace{-2pt}
At a high level, evolutionary algorithms generate samples and select which to keep for the next generation. MCMC's pseudocode is thus similar to that of certain evolutionary algorithms. In this section we show how some evolutionary algorithms can be interpreted as implementing approximate versions of MCMC, with a given evolutionary algorithm implicitly defining a transition kernel and a target distribution. While this interpretation will help us derive a class of evolutionary algorithms well suited for discovery, to build intuition we first demonstrate it using arguably the two simplest evolutionary algorithms, random search and hill climbing evolutionary strategies. The following derivations clarify which approximations the MCMC interpretation requires and how other, more general evolutionary algorithms can be interpreted as sampling from various target distributions as well.

In all cases, we assume the transition kernel uses a generative model $p(x)$ which can be conditioned using a prior sample, giving $p(x'|x)$. For example, this could be a large language model (LLM) which is asked to generate code and is conditioned on a previously generated program. For ease of notation, we omit the generative model's (fixed) parameters $\theta$, the sampling parameters (e.g. the temperature), and other constant factors from the conditioning.

\vspace{-3pt}
\subsection{Random Search}
In random search one samples IID from the generative model, conditioned on no prior samples. Thus, the transition kernel is $q(x'|x)=p(x')$. As random sampling always accepts the generated samples, its acceptance probability is always $1$, requiring $1=\frac{\pi(x')q(x|x')}{\pi(x)q(x'|x)}=\frac{\pi(x')p(x)}{\pi(x)p(x')}$. This is satisfied if $\pi(x)=p(x)$, which makes intuitive sense -- given no signal, random sampling's target is simply to sample from the proposal distribution. Interpreting random search as MCMC thus requires no approximations.

\vspace{-3pt}
\subsection{Hill Climbing Evolutionary Strategies}
There are two main kinds of hill climbing evolutionary strategies, $(1+N)$ and $(1,N)$ strategies \citep{beyer2002evolution}. The first, $(1+N)$, has $N$ offspring generated from the parent and picks the best out of the $N$ offspring+the one parent as the parent for the next generation. The second strategy, $(1,N)$, picks the best out of the $N$ offspring, regardless of the parent's fitness.

Although MCMC by default generates only a single sample per step, it is possible extending it to use $N$ samples per step. We extend MCMC to several samples using a trick used in some particle methods \citep{andrieu2010particle}, known as sampled-importance-resampling (SIR) \citep{rubin1987calculation}. First, conditioned on the previous sample $x$, hill climbing samples $N$ candidates from the generative model, yielding $x'_i\sim p(\cdot|x)$ for $1\leq i\leq N$. To simulate the best-candidate choice used in hill climbing, we soft sample the max-reward by resampling a $x'$ from the $\{x'_i\}_{i=1}^N$ by picking $x'_i$ with probability proportional to $R(x_i')^\beta$. Given a sufficiently high beta, this simulates the hard-max in hill climbing. Altogether, the sequential sampling-resampling process effectively yields the transition kernel $q(x'|x)\propto p(x'|x)R(x')^\beta$, with hill climbing requiring $\beta\to\infty$.

This transition kernel is equally valid for both $(1+N)$ and $(1,N)$ strategies. However, each strategy has a different acceptance criteria -- $(1+N)$ accepts a proposal $x'$ only if it is better than $x$, while $(1,N)$ accepts it regardless. Starting from the $(1,N)$ case, to always accept a sample requires the target distribution to fulfill the property $1=\frac{\pi(x')q(x|x')}{\pi(x)q(x'|x)}=\frac{\pi(x')p(x|x')R(x)^\beta}{\pi(x)p(x'|x)R(x')^\beta}$. As for generic $p(x'|x)$ the always-accept condition cannot be fulfilled, to satisfy it requires either approximating or assuming the generative model's conditional distribution is symmetric, so $p(x'|x)=p(x|x')$. Given a symmetric generative conditional distribution, to always-accept a sample requires $\pi(x)\propto R(x)^\beta$. Note that this is the distribution-matching target distribution proposed in section \ref{sec:dm4disc}, albeit at the infinite-beta limit.

For $(1+N)$, to reject a sample if $R(x)>R(x')$ requires $\frac{\pi(x)}{R(x)^\beta}\gg\frac{\pi(x')}{R(x')^\beta}$. In the infinite-beta limit, this is fulfilled by setting $\pi(x)\propto R(x)^{k\beta}$ for some $k>1$.

Thus, to summarize, hill climbing evolutionary strategies can be interpreted as MCMC with:
\begin{enumerate}
    \item A transition kernel $q(x'|x)\propto p(x'|x)R(x')^\beta$, where $\beta\to\infty$
    \item Assuming the generative model's conditional distribution is symmetric: $p(x'|x)=p(x|x')$
    \item A target distribution $\pi(x)\propto R(x)^{k\beta}$, where $k>1$ for $(1+N)$ strategies and $k=1$ for $(1,N)$
\end{enumerate}

\section{Distribution Matching Evolutionary Algorithms}
In the previous sections we discuss what is a suitable distribution matching objective for discovery and introduce MCMC, which is necessary for sampling from the target distribution using inference-only methods. Additionally, we show how some evolutionary algorithms can be interpreted as implementing approximate MCMC. In this section we bring these tools together to design a practical algorithm to sample from the target using a generative model. The following can be interpreted as deriving a form of hill climbing which a) does not require a symmetric transition kernel and b) samples from the proxy target distribution $\pi(x)\propto R(x)^\beta$ given finite beta.

Given the target distribution $\pi(x)\propto R(x)^\beta$, what should be the transition kernel? The optimal kernel is the target, where $q(x'|x)=\pi(x')$, but if it was possible to sample from the target then there would be no need to use MCMC in the first place. However, it is possible approximately sampling from the optimal transition kernel by sampling in batches using sampled-importance-resampling.

To illustrate, given $x_1,x_2,...,x_N\sim U$, where $U$ denotes the uniform distribution over the space, one can use batched importance sampling to sample from $\pi$ instead. By sampling $x'$ from $\{x_1,...,x_N\}$ with probability $\propto R(x_i)^\beta$, one effectively generates a sample from $\pi$. The sampling-resampling process will better approximate sampling from $\pi$ a) for larger $N$ and b) the closer $\pi$ is to $U$, with the notion of closeness being made precise later.

Sampling uniformly over the space would typically lead to very low quality samples, as most $x$ would have a very low reward. A much better initial proposal would be to use the generative model, so $x_1,...,x_N\sim p(\cdot|x)$. Due to sampling from $p(\cdot|x)$ and not $U$, $x'$ requires an importance sampling correction, so samples should be selected from $\{x_1,...,x_N\}$ with probability $\propto \frac{R(x_i)^\beta}{p(x_i|x)}$. This process yields the transition kernel
\begin{equation}
    q(x'|x)\approx p(x'|x) \frac{R(x')^\beta}{p(x'|x)Z}=\frac{R(x')^\beta}{Z}=\pi(x'),
\end{equation}
where the first $p(x'|x)$ stems from sampling $x_i\sim p(x_i|x)$ while the $\frac{R(x')^\beta}{p(x'|x)Z}$ term comes from sampling $x'$ from $\{x_1,...,x_N\}$ with probability $\propto \frac{R(x_i)^\beta}{p(x_i|x)}$. 

Note that as $q(x'|x)\approx \pi(x')$ the MCMC acceptance probability is approximately one. In practice, we assume the approximation is exact and drop the accept/reject step. 

Altogether, this transition kernel, target distribution, and sampling process result in Algorithm \ref{alg:dme}. Note that as we are chiefly interested in discovery problems, only the sample with the highest reward is returned.

\vspace{-6pt}
\begin{algorithm}
\label{alg:dme}
\caption{Distribution Matching Evolution}

\KwInput{A starting point $x_0$, a (potentially unnormalized) target distribution $\pi(x)$, a generative model $p(x'|x)$, batch size $N$, number of steps $n$}

\For{$i=0...n$}{
    Sample $x'_1,...,x'_N\sim p(\cdot|x_i)$
    
    Define the importance weights $w_i\coloneqq \frac{R(x_i')^\beta}{p(x'_i|x_i)}$ and self-normalized probabilities $p_i\coloneqq \frac{w_i}{\sum_{j=1}^N w_j}$

    Sample $x'\sim \{x'_1,...,x'_N\}$, where $x_i'$ is sampled with probability $p_i$
  
    Set $x_{i+1}\leftarrow x'$
}

Return the generated sample with the highest reward

\end{algorithm}

\vspace{-18pt}
\subsection{Quality of Approximation}
\vspace{-3pt}
The sampling-resampling process DME uses only approximates sampling from the true target distribution due to using a finite batch size. To understand how well sampling from the target is approximated we derive the first order correction to $q(x'|x)$ in the sampling batch size, $\frac{1}{N}$. Broadly, it is clear that if the generative model exactly samples from the target, so $p(x'|x)=\pi(x')$, there is no approximation, and in some sense the closer these two distributions are the better is the approximation. The size of this error is given by the first order correction, with proofs and mild necessary assumptions being delegated to Appendix \ref{app:proofs}.

\begin{theorem}[First order correction to $q(x'|x)$, informal]
\label{thm:first_corr}
    Define the transition kernel $q(x'|x)$ using the sampling-resampling process given in Algorithm \ref{alg:dme}. Then, under some mild conditions, one has that
    \begin{equation}
        q(x'|x)=\pi(x')\left(1+\frac{D(x'|x)}{N}\right)+O\left(N^{-3/2}\right),
    \end{equation}
    where $D(x'|x)\coloneqq \int\frac{\pi(y)^2}{p(y|x)}dy-\frac{\pi(x')}{p(x'|x)}$.
\end{theorem}

Note that $D(x'|x)=0$ if $\pi(x')=p(x'|x)$, as required for the sampling process to be unbiased. To get a measure for the correction $D(x'|x)$'s magnitude, one can measure the average error when sampling proposals $x'$ from $p(\cdot|x)$ instead of from $\pi$. Doing so yields the following corollary, which quantifies this error.

\begin{corollary}[First order correction as a $\chi^2$ divergence]
    \label{corr:chisq}
    One has that
    \begin{equation}
        \mathbb{E}_{x'\sim p(\cdot|x)}[D(x'|x)]=\int \left(\frac{\pi(x')-p(x'|x)}{p(x'|x)}\right)^2 p(x'|x)dx'=\chi^2(\pi||p(\cdot|x))
    \end{equation}
    Where $\chi^2$ is the chi-squared divergence. Specifically, note that $\chi^2$ has the form of a weighted mean-square error between the two distributions.
\end{corollary}

Thus, there are three ways to improve the approximate sampling from the target distribution:
\begin{enumerate}
    \item Use an easier target which is more similar to the generative model's distribution. Typically, this would mean lowering $\beta$.
    \item Use a larger batch size $N$ when sampling proposals.
    \item Use a better generative model, or improve the given model's conditioning. For example, when the generative model is a language model, a better prompt would focus more probability mass on high-reward regions.
\end{enumerate}

MCMC's sequential sampling compensates for the transition kernel being only approximately optimal, where a better approximation requires fewer steps to effectively sample the target.

\subsection{Novelty, Heuristic Interpretations, and Proposal-Agnosticism}
To better understand DME it helps interpreting its behavior in some limits, in which cases it reduces to various heuristics. For example, in the $\beta\to\infty$ limit DME would always select the highest reward sample, akin to hill climbing. Where it differs from typical hill climbing is in how ties are resolved given several samples with the same reward. If $R(x_1)=R(x_2)$ but $p(x_1|x)<p(x_2|x)$, DME will with higher probability pick $x_1$ due to the inverse dependence on the probability in the importance weights. In this sense DME prefers novel solutions, where lower probability samples are always picked all else being equal. If additionally one of the high-reward samples has a much lower probability than all the others, then DME is equivalent to hill climbing with a pick-lowest-probability tie-breaker. More generally, for finite $\beta$, DME can be interpreted as a form of stochastic hill climbing.

An evolutionary algorithm typically cannot be described independently of its mutation operator, as different kinds of mutations can make the same algorithm give very different results. For example, hill climbing with one model proposing candidates may pick approach A $30\%$ of the time and approach B $70\%$, while for a different model the split may instead be $90\%/10\%$. Unlike typical evolutionary algorithms DME is approximately proposal-agnostic, as the importance sampling weights correct for the probability the generative model gives to different samples. Due to the approximation's quality, in practice DME may still be much more sample efficient with some models than with others.

\section{DME for Language Models}
While DME as described in the previous section is suitable for generic generative models, in this section we describe two tweaks that make it better suited for language models. The general algorithm remains the same, with the two tweaks resulting only in modifying the importance sampling weights. Other design choices, for language models or other generative models, would accordingly give a different valid distribution matching evolutionary algorithm.

\subsection{Path Marginalization}
While there are several ways to arrive at the same answer, typically the answer and not the path that led to it is what matters. For example, different reasoning trajectories may result in a language model writing the same code. If $x$ is the previous code and $x'$ is the new code, although only the probability of generating the final code is of interest, the model on the way also generates some reasoning trajectory, $z$, yielding $p(x',z|x)$ and not $p(x'|x)$. More generally, $z$ can be any intermediate steps towards the final answer, e.g. a set of tool calls and their returned outputs. The probability of arriving at the sampled code given the prior code, regardless of what the intermediate steps are, can be approximated as
\begin{equation}
    p(x'|x)=\underbrace{\int p(x'|z,x)p(z|x)dz}_{\mathbb{E}_{z\sim p(\cdot|x)}[p(x'|z,x)]}\approx p(x'|z,x),
\end{equation}
where the approximation is done by taking a single Monte Carlo sample from the expectation. Effectively, this means that only the probability over the relevant part of the output should be used, not the probability over the entire output. Intuitively, a mundane program generated following a very unlikely reasoning trajectory should have a reasonable generation probability, instead of being artificially down-weighted due to the unlikely reasoning that led to the program.

\subsection{Length-Respecting Uniform Prior}
When deriving DME we initially assumed the batch's samples are sampled from a uniform distribution over the space, hence the need for the $1/p(x'|x)$ importance sampling correction. However, which uniform distribution to use was not explicitly specified. Na\"ively selecting a uniform distribution over sequences artificially up-weighs long sequences, e.g. for a vocabulary of size $V$ there are $V$ sequences with one token but $V^2$ with two. To reduce this length dependence for problems where shorter solutions are likely preferred, we define a uniform prior using the following process.

Assume tokens are sampled uniformly from the vocabulary, and a sequence ends when an end-of-sequence token is sampled, and assume there is one such token in the vocabulary. Thus, the probability of sampling the token sequence making up $x$ is $U(x)=\frac{1}{V^{|x|}}$, as each token in the sequence has probability $1/|V|$ of being sampled. 

\subsection{Modified Importance Sampling Weights}
Altogether, these two modifications change the importance weights as follows. The path marginalization results in using $w=\frac{R(x')^\beta}{p(x'|z,x)}$ instead of $w=\frac{R(x')^\beta}{p(x'|x)}$, while the length-respecting uniform prior changes the numerator from $w=\frac{R(x')^\beta}{p(x'|z,x)}$ to $w=\frac{R(x')^\beta U(x')}{p(x'|z,x)}=\frac{R(x')^\beta}{V^{|x'|}p(x'|z,x)}$. Note that the previous importance sampling weights could omit the $U(x)$ term as it was constant.

The path marginalization also yields a different heuristic interpretation of DME. When $\beta\to\infty$ and the generation probabilities between different samples are very different, DME simply implements a tie-breaker for hill climbing, where the highest reward sample with the lowest probability should be picked. The path marginalization results in the probability being calculated only over the relevant part of the answer instead of the entire generated sequence. Unless specified otherwise, the path marginalization is always used by default, while the default prior is the regular uniform distribution.

\section{Experiments}
In this section we test DME's sample efficiency on several problems. Specifically, we ask what is the probability DME and other methods can solve a problem given some number of samples. For difficult discovery problems requiring many samples any method which mode collapses or struggles sampling low probability solutions will have a low probability of success.

\subsection{Regular Expressions}
\label{sec:regex}
To facilitate comparing DME to many existing methods we begin by testing its performance in a small-scale setting, generating regular expressions (regex). A regex is a short expression which accepts or rejects strings, akin to how a program can pass or fail unit tests. Regex are relatively high entropy strings, making them a good test case for rare event sampling.\footnote{For example, a regex that accepts gmail addresses is \texttt{\string^\string[\string\w.+\string\-]+\string@gmail\string\.com\$} \citep{stelmaszczyk2013gmailregex}.} A sampled regex can count as a success if it accepts a set of positive example strings (e.g. email addresses) and rejects a set of negative examples (e.g. malformed addresses).

This setup is designed to be relatively simple and quick to run, enabling reliable probability estimates when comparing different methods. Specifically, we use Qwen3-0.6B as the generative model for all search methods and cap its output to a maximum of 16 tokens, due to the regex being relatively short. All methods are given 100k samples to find the regex and run on the same hardware -- a single TitanRTX GPU -- to facilitate runtime comparisons.

The following methods are compared. Note that most of these methods perform local reward maximization and accordingly have objectives which permit mode collapse. First, we test three gradient-based discovery methods. Risk-sensitive GRPO (\textbf{RS-GRPO}) is a reinforcement-learning method with a diversity-inducing objective \citep{jiang2025risk}. Test-time training (\textbf{TTT-Discover}) is a more elaborate RL-based discovery framework \citep{yuksekgonul2026learning}. \textbf{GFlowNets} is a framework for optimizing similar distribution-matching objectives using methods from variational inference and reinforcement learning \citep{malkin2022gflownets,malkin2022trajectory,tiapkin2024generative}. The remaining five baselines and DME are all gradient-free. In Best-of-N (\textbf{BoN}), also known as random sampling, the model is sampled IID while prompted to find the regex. Sequential conditioned sampling (\textbf{SCS}) is a strong baseline for discovery methods which uses a meta-model (also Qwen3-0.6B) to propose new prompts to find the regex \citep{gideoni2026simple}. \textbf{GEPA} is a similar, more sophisticated method which proposes prompts also based on returned text-based task-specific information \citep{agrawal2026gepa}. \textbf{(1+N)} and \textbf{(1,N)} are two standard hill climbing evolutionary strategies. \textbf{DME} is run with $\beta=1000$. All methods had their setups and hyperparameters tuned, use a batch size of 64 where relevant, and where possible use the same prompts. Methods that rely on a reward signal use $e^{\text{accuracy}}$ as their reward function, where the accuracy is the percent of test cases which passed. Additional details are in Appendix \ref{app:regex}.

For the test data we use the regex dataset of \citet{siddiq2024re} and filter to 20 regexes for which best-of-N gets between a 5\% and 50\% success rate, to ensure we focus on regexes which are difficult but solvable. All methods are run 10 times per regex except for some especially slow baselines, which are run only 3 times. We report the average success rate across regexes and 95\% MOVER-aggregated confidence intervals \citep{zou2008construction}.

\begin{figure}[!h]
    \centering
    \vspace{-5pt}
    \includegraphics[width=\linewidth]{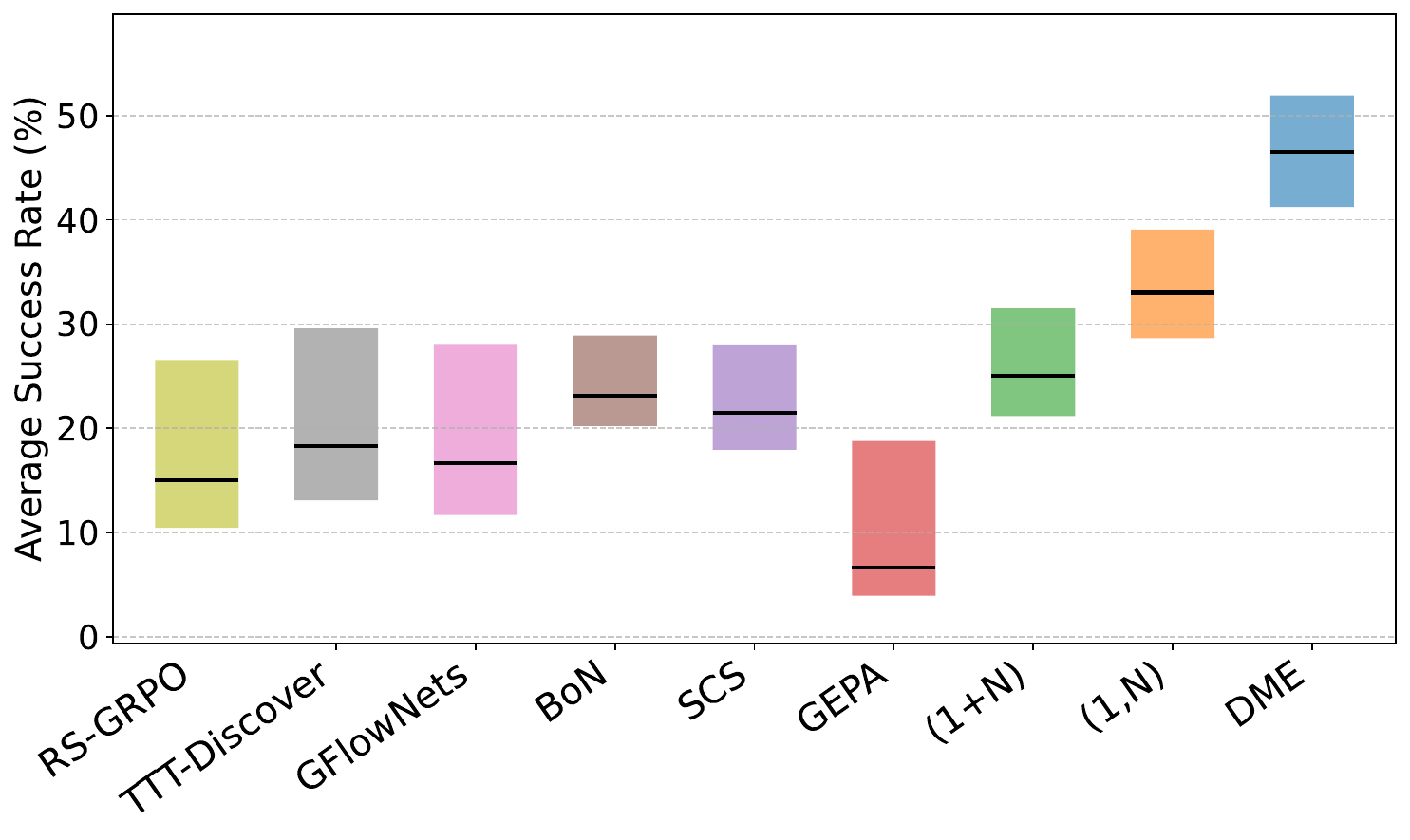}
    \vspace{-15pt}
    \caption{Average success rate across regexes for different methods. Black lines denote means, shaded regions are 95\% confidence intervals.}
    \vspace{-15pt}
    \label{fig:regex}
\end{figure}

\begin{table}[!h]
  \centering
  \label{tab:method_comp}
  \resizebox{\textwidth}{!}{%
  \begin{tabular}{lccccccccc}
      \toprule
       & RS-GRPO & TTT-Discover & GFlowNets & BoN & SCS & GEPA & (1+N) & (1,N) & DME \\
      \midrule
      Time per run (minutes)    & $154.8 \pm 41.5$ & $324.3 \pm 56.4$ & $238.8 \pm 92.6$ & $3.0 \pm 0.4$ &
      $18.1 \pm 2.9$ & $47.5 \pm 14.9$ & $6.7 \pm 1.9$ & $6.5 \pm 0.9$ & $7.2 \pm 0.5$ \\
      Markovian    & \checkmark &            & \checkmark & \checkmark
      & \checkmark &            & \checkmark & \checkmark & \checkmark \\
      Requires a meta-model                                         &            &            &            &
      & \checkmark & \checkmark &            &            &            \\
      Requires weight updates   & \checkmark & \checkmark & \checkmark &
      &            &            &            &            &            \\
      Requires log-probs        & \checkmark & \checkmark & \checkmark &
      &            &            &            &            & \checkmark \\
      \bottomrule
  \end{tabular}%
  }
\caption{Properties of different methods. DME is relatively fast, Markovian, does not need auxiliary models, and is inference only. However, it does require access to the generative model's (potentially unnormalized) density. Runtime uncertainties are standard deviations.}
\end{table}

\vspace{-3pt}
\subsubsection{Results}
Results across methods are shown in Figure \ref{fig:regex}. DME gets the highest success rate, almost  twice the success rate of best-of-N. $(1,N)$ hill climbing is perhaps the only method other than DME to perform notably better than best-of-N, being effectively a simplified greedy heuristic version of DME in the $\beta\to\infty$ limit. We speculate $(1,N)$ hill climbing outperforms $(1+N)$ due to always refreshing the regex used as hint for generating the next batch, thus making its search more varied. $(1+N)$ changes its hint only if it generates an improvement over the best regex seen so far.

It is surprising that GFlowNets have a low success rate given they are the only method other than DME which uses a distribution matching objective. We speculate GFlowNets' low success is due to the gradient-based methods being unstable and mode collapsing even given distribution matching objectives. Empirically, the three tested baselines all work best only given very low learning rates -- $10^{-7}$ for TTT-Discover and GFlowNets, $10^{-6}$ for RS-GRPO -- with higher learning rates leading all three to mode collapse.

\vspace{-3pt}
\subsection{Coding Competitions}
For a more difficult, real-world task, we measure how well DME can solve problems from coding competitions. Specifically, we measure to what extent DME can enable a model to find correct answers which would otherwise require many samples to generate.

We take a random 300 question subset of LiveCodeBench \citep{jain2025livecodebench} and compare DME to $(1,N)$ and $(1+N)$ evolutionary strategies as they performed best for the regex task. DME uses $\beta=1000$ and the length-respecting prior, as coding competition problems often have succinct solutions -- see Appendix \ref{app:code_comp} for ablations and details. All methods have a batch size of 8. The proposal model is Qwen2.5-Coder-1.5B Instruct \citep{hui2024qwen2}. As in the regex case, the reward is $e^\text{accuracy}$, where a given solution's accuracy is the percent of unit tests that pass.

Results are shown in Figure \ref{fig:lcb}. DME is empirically more sample efficient than the evolutionary strategies, requiring fewer programs to solve the same fraction of problems. For example, DME reaches a $\sim30\%$ solve rate given $50$ programs, whereas both evolutionary strategies require generating $\sim100$ programs per question to reach the same performance.

\begin{figure}[!h]
    \vspace{-8pt}
    \centering
    \includegraphics[width=0.7\linewidth]{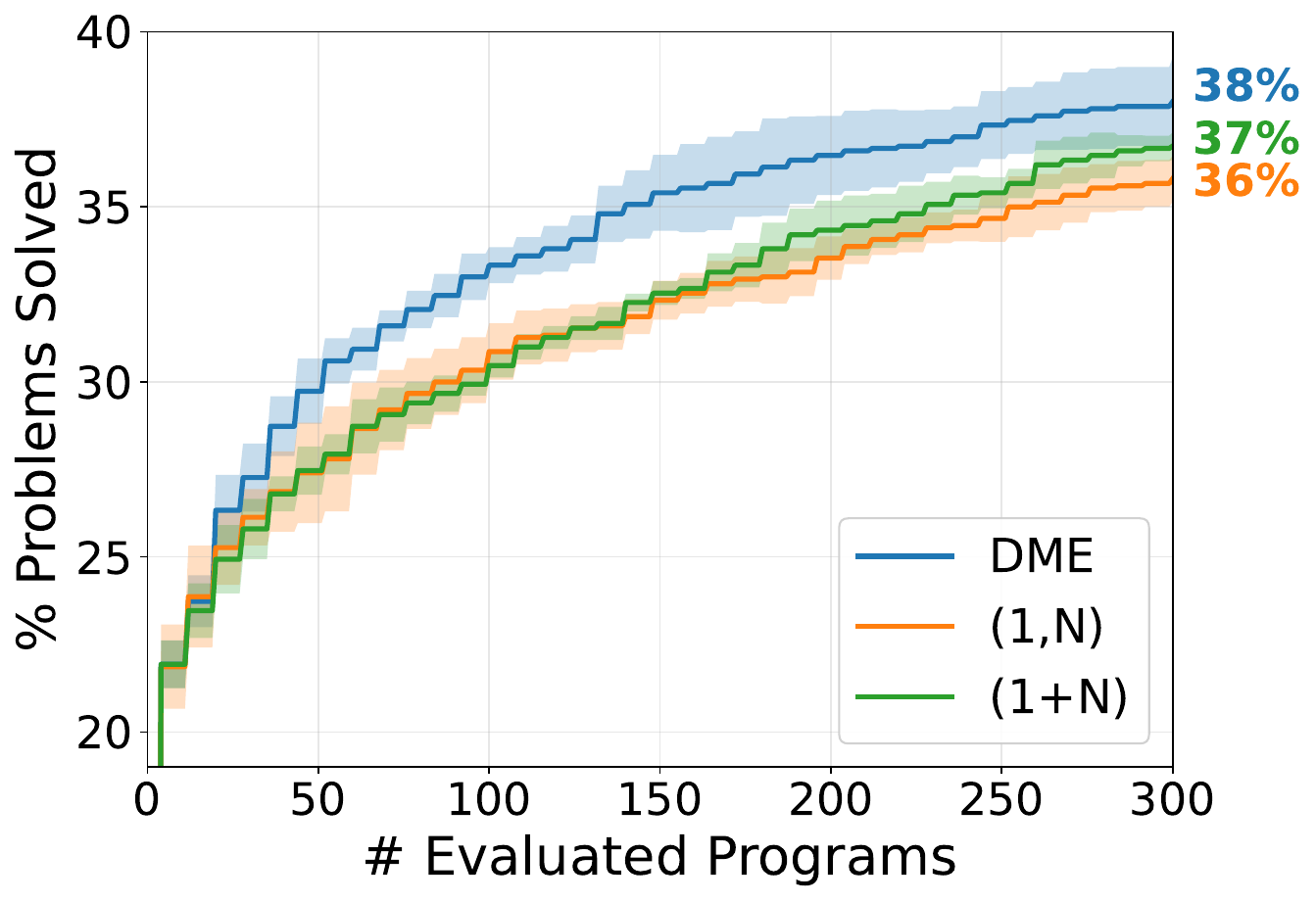}
    \vspace{-15pt}
    \caption{Different methods' solve-rate for competitive coding problems as a function of the number of generated programs per question. Shaded regions denote the standard deviation over 5 seeds.}
    \label{fig:lcb}
    \vspace{-15pt}
\end{figure}

\section{Math Bounds}
\label{sec:math_bounds}
AlphaEvolve \citep{novikov2025alphaevolve} demonstrated how code generation can be used to discover new mathematical bounds. For certain problems, finding a bound is analogous to a black-box optimization problem, where a verifier gets as input a list of numbers and outputs a corresponding mathematical bound. A language model can then either generate code which directly constructs this list of numbers or finds it using various optimization methods. 

We test DME's sample efficiency on 9 of the problems analyzed in \citet{novikov2025alphaevolve}, comparing it to two strong open-source baselines benchmarked by \citet{gideoni2026simple}. The first baseline is best-of-N (BoN), while the second is a more sequential version thereof called sequential-conditioned-sampling (SCS). \citet{gideoni2026simple} found SCS to be a strong baseline when compared to open-source versions of AlphaEvolve, like ShinkaEvolve \citep{lange2026shinkaevolve}. 

All methods are given a sample budget of 5000 programs, sampling from GPT-OSS 20B \citep{agarwal2025gpt}. GPT-OSS' knowledge cutoff is about a year before AlphaEvolve came out, so \citet{novikov2025alphaevolve}'s solutions cannot be memorized. SCS hyperparameters are kept at their proposed defaults. DME uses $\beta=10^6$ and to reduce runtime uses four chains or, as they are known in evolutionary algorithms parlance, four islands. All methods use the same prompt templates, and a temperature and top-$p$ of 1. Following \citet{gideoni2026simple}, methods are considered tied if their results are within floating point errors as per \texttt{numpy.isclose}.

\begin{table}[h]
    \centering
    \vspace{-6pt}
    \begin{tabular}{L{4.4cm} c c c}
        Problem & BoN & SCS & DME \\
        \midrule
        \# problems DME $\geq$ method &  9/9 & 9/9 &  \\
        Average rank \arrowtag{\downarrow}  & 2.6 & 1.8 & 1.0 \\
        \midrule
        First autocorr. ineq. \arrowtag{\downarrow}
            & 1.524
            & 1.517
            & \textbf{1.510} \\

        Second autocorr. ineq. \arrowtag{\uparrow}
            & 0.8779
            & \underline{0.8784}
            & \textbf{0.8868} \\

        Uncertainty ineq. \arrowtag{\downarrow}
            & \textbf{0.3521}
            & \textbf{0.3521}
            & \textbf{0.3521} \\

        Erd\H{o}s' min. overlap \arrowtag{\downarrow}
            & 0.3815
            & \underline{0.3812}
            & \textbf{0.3809} \\

        Sums/diffs of sets \arrowtag{\uparrow}
            & 1.110
            & \underline{1.114}
            & \textbf{1.117} \\

        Max--min dist. ratio \arrowtag{\downarrow}
            & \textbf{12.8892}
            & \textbf{12.8892}
            & \textbf{12.8893} \\

        Heilbronn triangles \arrowtag{\uparrow}
            & 0.0339
            & \underline{0.0360}
            & \textbf{0.0365} \\

        Kissing number in 11D \arrowtag{\uparrow}
            & 348
            & \underline{566}
            & \textbf{582} \\

        Circle packing \arrowtag{\uparrow}
            & 2.62304
            & \underline{2.63109}
            & \textbf{2.63598}
    \end{tabular}
    \vspace{-5pt}
    \caption{Best discovered bounds by different methods when given 5000 samples. DME either finds the best bound or ties for first place for all problems. Arrows denote whether higher/lower is better. For all problems except circle packing results are truncated to the first different significant digit. Best results are bolded, second best are underlined.}
    \label{tab:results_math}
    \vspace{-5pt}
\end{table}

Table \ref{tab:results_math} shows the results. Relative to the two baselines, DME finds the best bound for all problems, either finding the best outright or tying for first place with a different method.

To see how DME performs over a range of budgets we compute the probability of improvement between DME and the two baselines as a function of the number of generated programs. This comparison asks ``when picking a random problem, what is the probability DME will outperform a given method?'', with ties counting as each method having an equal probability of winning \citep{agarwal2021deep}. Figure \ref{fig:prob_imp} shows that for almost all sampling budgets DME is likely to improve over both baselines.

\begin{figure}[h]
    \centering
    \vspace{-6pt}
    \includegraphics[width=0.7\linewidth]{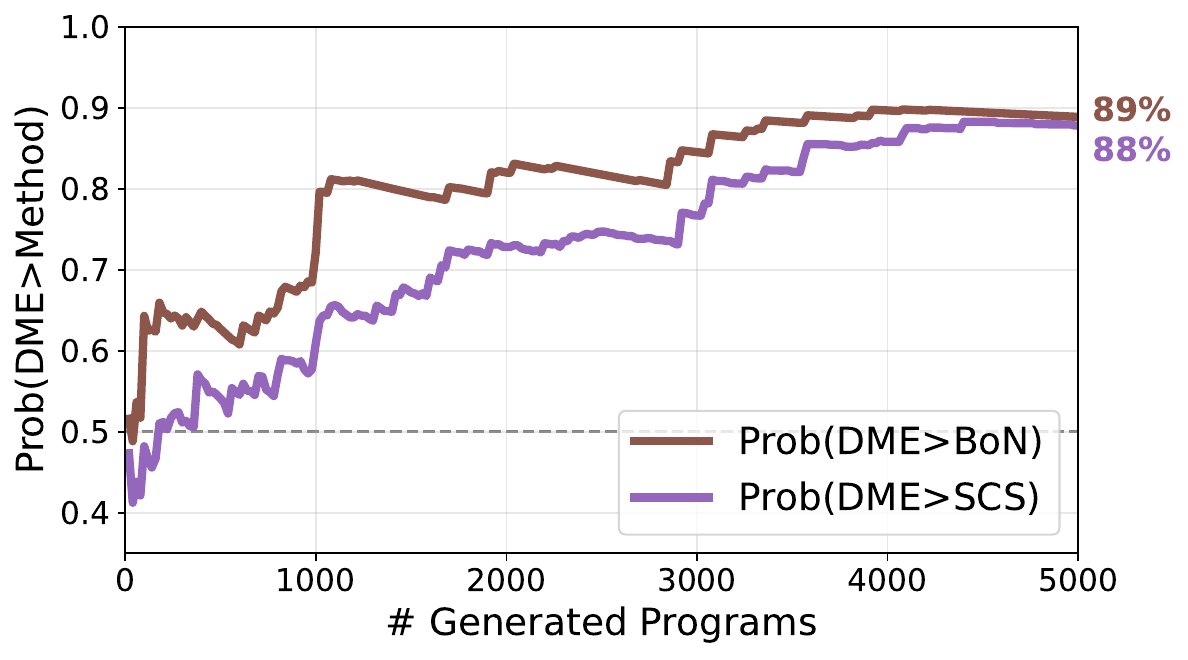}
    \vspace{-10pt}
    \caption{DME's probability of improvement over both baselines as a function of the number of generated programs. For most sample budgets DME is likely to outperform both baselines. For per problem breakdowns see Appendix \ref{app:math_per_problem}.}
    \label{fig:prob_imp}
    \vspace{1pt}
\end{figure}

\vspace{-3pt}
\section{Related Work}
DME is an MCMC-based method to optimize distribution-matching objectives. These same objectives can be directly optimized using variational inference or reinforcement learning if weight updates are acceptable. To illustrate, given a target distribution $\pi$ and a parameterized distribution $p_\theta$, the minimum of
\begin{equation}
    \text{KL}(p_\theta||\pi)=\mathbb{E}_{p_\theta}[\log p_\theta-\log\pi]=\mathbb{E}_{p_\theta}[-\log\pi]-H(p_\theta)
\end{equation}

is achieved when $\pi=p_\theta$, where $H(p)$ denotes $p$'s entropy. Note that akin to DME, this objective does not require $\pi$ to be normalized. When substituting the DME target distribution $\pi(x)\propto R(x)^\beta$ and taking the KL's gradient shows how optimizing this objective is equivalent to a certain form of soft reinforcement learning. Explicitly, with the DME target distribution, $\text{KL}(p_\theta||\pi)=-\beta\mathbb{E}_{p_\theta}[\log R(x)]+\log Z-H(p_\theta)$. Rescaling by $\frac{1}{\beta}$ and setting $R'(x)\coloneqq \log R(x),\alpha\coloneqq \frac{1}{\beta}$, one finds that:
\begin{equation}
    -\nabla_\theta \alpha \text{KL}(p_\theta||\pi)=\mathbb{E}_{p_\theta}[R'(x)\nabla_\theta \log p_\theta]+\alpha\nabla_\theta H(p_\theta),
\end{equation}

which is exactly a soft reinforcement learning policy gradient \citep{ziebart2010modeling, levine2018reinforcement}.

There are various ways to optimize $\text{KL}(p_\theta||\pi)$ or other distribution matching objectives. The most closely related method to DME is GFlowNets, which propose a variety of losses to train a generative model to sample proportionally to some reward \citep{bengio2021flow,malkin2022trajectory,madan2023learning}. Beyond requiring weight updates, GFlowNets differ from DME in that some formulations require only samples from the generative model and not densities, the sampling is amortized, and in some cases GFlowNets are known to be more sample efficient than MCMC based methods \citep{bengio2021flow}. Some other rare event sampling methods, such as the cross-entropy method \citep{rubinstein1997optimization}, optimize $\text{KL}(\pi||p_\theta)$. Surprisingly, the cross-entropy method was found to perform worse than other gradient-based methods for the regex experiments in section \ref{sec:regex}, leading to its exclusion.

DME is also related to some symbolic code optimization methods. Superoptimization is the process of finding an optimal computer program for some task using exhaustive search methods \citep{massalin1987superoptimizer}. \citet{schkufza2013stochastic} propose stochastic superoptimization, a method which uses MCMC over a finite instruction set to more efficiently find an optimal program. DME is in a sense \textit{generative} stochastic superoptimization, as it similarly uses MCMC to find optimal programs, while relying on generative models to increase the sample efficiency further.

DME is related to various MCMC algorithms. Most closely related is iterated-sampling-importance-resampling (i-SIR), which similarly samples $N$ candidates and then resamples a proposal from them weighted by some importance sampling weights \citep{tjelmeland2004using,andrieu2010particle}; for an overview, see \citet{laitinen2025iterated}. i-SIR was originally proposed to more efficiently calculate expectations relative to traditional Metropolis-Hastings MCMC, by using both rejected and accepted candidates \citep{tjelmeland2004using}. Similarly, upon completion DME can return a sample that was originally rejected if it has the highest seen reward.

Many prior works using generative models have found discoveries using various heuristic evolutionary algorithms \citep{romera2024mathematical,novikov2025alphaevolve,openai2026tenadvances}. These discoveries are likely enabled by a combination of using strong underlying generative models and generating many samples, thus sampling lower probability completions. Supporting this claim, \citet{gideoni2026simple,tsoukalas2026advancing} find that simple search heuristics perform comparably well to more complicated evolutionary pipelines given the same generative model and sample budgets. \citet{gideoni2026simple} further note that a problem's formulation can make it more amenable to automated search, so a better problem formulation can make a method more sample efficient.

\section{Discussion}
There are several straightforward extensions to DME. Incorporating other methods from either the MCMC or evolutionary algorithms literature could prove fruitful. If using gradients is permissible, it would be interesting using Hamiltonian Monte Carlo as a potentially much more efficient variant of DME. Another interesting approach would be controlling to what extent DME upweighs low probability samples by tempering the generative model's probability in the importance weights. This could be easily achieved by setting $p\to p^\alpha$ for some $0<\alpha\leq1$. All this is left for future work.

It would be interesting to test distribution matching methods further on cases with noisy, surrogate, or otherwise inexact rewards. Surrogate rewards are especially important when a real experiment is either expensive or takes a long time, in which case discovery methods can help proposing potential experiments. The expensive to evaluate, real reward might be maximized by a sample which seems suboptimal given a cheap proxy. In this case it is important to propose a diverse set of candidates which have a high but not only a maximal reward per the proxy. While local reward-maximization methods would reward-hack the proxy, distribution matching methods can be tuned to sample per various reward distributions.

One limitation of DME is it requiring (unnormalized) probability densities. For some generative models finding the density is expensive, i.e. diffusion/flow matching models require at least an extra forward pass per step for calculating a generation's density \citep{grathwohl2018scalable}. Additionally, not all closed-source language models return their log-probabilities. While likely possible adapting DME for these settings, doing so would require deriving a new transition kernel and be potentially nontrivial.

More broadly, all training-free methods have some clear limitations. Any discovery which has a sufficiently low probability, likely due to being extremely dissimilar to the model's training data, will still require a potentially unreasonable number of samples to find. Practically, training-free discovery methods are better thought of as ways to expand what a model can find given some number of samples, with new discoveries then being added to the model's training data. A later finetuned or retrained generative model will then have an effectively larger set of reachable discoveries. Such an iterative procedure of expanding the set of reachable discoveries is sometimes described as how science progresses \citep{kaufman2000investigations}.

\subsection*{Acknowledgements}
We would like to thank Dulhan Jayalath and Lazar Đoković for giving feedback on an earlier version of this work. We thank Thinking Machines for a generous grant that enabled the math bounds experiments. Yonatan is funded by the Rhodes Trust and the AIMS EPSRC CDT (grant no. EP/S024050/1).

\bibliographystyle{plainnat}
\bibliography{bib}

\newpage

\appendix

\section{Proofs}
\label{app:proofs}

The formal version of theorem \ref{thm:first_corr}, including the necessary assumptions on the various distributions, is
\setcounter{theorem}{0}
\begin{theorem}[First order correction to $q(x'|x)$]
    Define the transition kernel $q(x'|x)$ using the sampling-resampling process given in Algorithm \ref{alg:dme}. Assume that $p(x'|x)$ has full support, that for all $x$ the $\chi^2$-divergence $\chi^2(\pi||p(\cdot|x))=\int \left(\frac{\pi(x')-p(x'|x)}{p(x'|x)}\right)^2 p(x'|x)dx'$ is finite, and that the importance weight's third absolute moment is finite, so $\mathbb{E}_{-i}[|w_j-Z|^3]<\infty$. Then, one has that
    \begin{equation}
        q(x'|x)=\pi(x')\left(1+\frac{D(x'|x)}{N}\right)+O\left(N^{-3/2}\right),
    \end{equation}
    where $D(x'|x)\coloneqq \int\frac{\pi(y)^2}{p(y|x)}dy-\frac{\pi(x')}{p(x'|x)}$.
\end{theorem}
\begin{proof}[Proof of Theorem \ref{thm:first_corr}, First order correction to $q(x'|x)$]
    We wish to find an exact analytic form for $q(x'|x)$ based on the process used in DME, i.e. Algorithm \ref{alg:dme}. For simplicity we denote $x_i$ simply as $x$. Given $\{x'_1,...,x'_N\}$, the probability of sampling $x'_1$ given each sample has weight $w_i\coloneqq \frac{R(x_i')^\beta}{p(x'_i|x)}$ is, as before, $p_i\coloneqq \frac{w_i}{\sum_{j=1}^N w_j}$. Thus, as $x'_i\sim p(\cdot|x)$, the probability of sampling $x'_i$ from the joint sampling process is $Np_i p(x'_i|x)$, where the $N$ factor comes from the $N$ samples being exchangeable. Note that $p_i$ implicitly still relies on the entire set $\{x_1,...,x_N\}$. Throughout the proof we will take many expectations of the form $\mathbb{E}_{x_1,...,x_{i-1},x_{i+1},...,x_N\sim p(\cdot|x)}[\cdot]$. For convenience, denote $\mathbb{E}_{-i}\coloneqq \mathbb{E}_{x_1,...,x_{i-1},x_{i+1},...,x_N\sim p(\cdot|x)}$.
    
    Denoting $x'=x'_i$ and marginalizing over the remaining $x'_j$ gives the exact $x\to x'$ transition probability:
    \begin{equation}
    \label{eqn:messy_trans_kernel}
        q(x'|x)=Np(x'|x)\mathbb{E}_{-i}[p_i]=NR(x')^\beta \mathbb{E}_{-i}\left[\frac{1}{w_i+\sum_{j=1,j\not=i}^N w_j} \right],
    \end{equation}
    where the expectation is over all $j\not=i$. This form of the transition probability is difficult to use as the expectation is over a fraction. To reach an analytic formula we wish to express the fraction as an expansion, so the expectation can act on polynomial terms of $w_j$. As a prologue to such an expansion, note that:
    \begin{equation}
        \mathbb{E}_{-i}[w_j]=\int \frac{R(x'_j)^\beta}{p(x'_j|x)}p(x'_j|x)dx'_j = \int R(x'_j)^\beta dx'_j=Z,
    \end{equation}
    where recall that $Z$ is the normalization constant of the target distribution, $\pi(x)$. Thus, denote $S\coloneqq(N-1)Z$ and $\Delta\coloneqq \sum_{j=1,j\not=i}^N w_j-S$. Note that $\mathbb{E}_{x'_j\sim p(\cdot|x),j\not=i}[\Delta]=0$, by design. The rest of the proof boils down to straightforward if somewhat tedious algebra and a truncated Taylor expansion around $\frac{\Delta}{w_i+S}=0$. 
    
    Using the new notation, one has that:
    \begin{equation}
        \mathbb{E}_{-i}\left[\frac{1}{w_i+\sum_{j=1,j\not=i}^N w_j} \right]=\mathbb{E}_{-i}\left[\frac{1}{w_i+S+\Delta} \right]=\frac{1}{w_i+S}\mathbb{E}_{-i}\left[\frac{1}{1+\frac{\Delta}{w_i+S}} \right].
        \label{eqn:exp_pre_expand}
    \end{equation}
     Now, note that as $w_i+S+\Delta=\sum_{j=1}^N w_j>0$ then the denominator is always positive and the expectation has no poles. Taking an exact finite Taylor expansion around $\frac{\Delta}{w_i+S}=0$ gives:
     \begin{equation}
     \label{eqn:delta_taylor_exp}
         \left(1+\frac{\Delta}{w_i+S}\right)^{-1}=1-\frac{\Delta}{w_i+S}+\left(\frac{\Delta}{w_i+S}\right)^2-\frac{\left(\frac{\Delta}{w_i+S}\right)^3}{1+\frac{\Delta}{w_i+S}},
     \end{equation}
     where we used the expansion $(1+x)^{-1}=1-x+x^2-\frac{x^3}{1+x}$ \citep[Eq.~(1.1.8)]{atkinson2008introduction}. Plugging the first two terms back into the expectation in equation \ref{eqn:exp_pre_expand} gives:
     \begin{equation}
         \mathbb{E}_{-i}\left[1\right]=1, \quad 
         \mathbb{E}_{-i}\left[\frac{\Delta}{w_i+S}\right]=0,
     \end{equation}
     as $\mathbb{E}_{-i}[\Delta]=0$. For the third term, we have that:
     \begin{equation}
     \label{eqn:exp_delta_sq}
         \mathbb{E}_{-i}\left[\Delta^2\right]
         =\mathbb{E}_{-i}\left[\sum_{j,k=1,j,k\not=i}^Nw_jw_k\right]-S^2.
     \end{equation}
     For the $(N-1)^2-(N-1)$ terms where $j\not=k$ one has that $\mathbb{E}_{-i}\left[w_jw_k\right]=\mathbb{E}_{-i}\left[w_j\right]\mathbb{E}_{-i}\left[w_k\right]=Z^2,$
    which leaves only the $N-1$ remaining $\mathbb{E}_{-i}\left[w_j^2\right]$ terms. Note that:
    \begin{equation}
    \label{eqn:exp_w_sq}
        \mathbb{E}_{-i}\left[w_j^2\right]
        =\int \frac{R(x'_j)^{2\beta}}{p(x'_j|x)^2}p(x'_j|x)dx'_j=Z^2\int \frac{\pi(y)^2}{p(y|x)}dy,
    \end{equation}
    where for convenience the integrand is denoted as $y$ instead of $x'_j$. Plugging these results back into equation \ref{eqn:exp_delta_sq} gives:
    \begin{equation}
        \begin{aligned}
            \mathbb{E}_{-i}\left[\Delta^2\right]
            &= \left((N-1)^2-(N-1)\right)Z^2+(N-1)\mathbb{E}_{-i}\left[w_j^2\right]-S^2 \\
            &= (N-1)Z^2\left(
                N-2+\int \frac{\pi(y)^2}{p(y|x)}\,dy-(N-1)
            \right) \\
            &= (N-1)Z^2\left(
                \int \frac{\pi(y)^2}{p(y|x)}\,dy-1
            \right)
        \end{aligned}
    \end{equation}
    Now, temporarily denoting $I\coloneqq \int \frac{\pi(y)^2}{p(y|x)}\,dy-1$, one has that the expectation of the second order term in equation \ref{eqn:delta_taylor_exp} is:
    \begin{equation}
    \label{eqn:exp_second_order}
        \mathbb{E}_{-i}\left[\left(\frac{\Delta}{w_i+S}\right)^2\right]=\frac{(N-1)Z^2I}{\left(w_i+S\right)^2}.
    \end{equation}
    Specifically, note that $\chi^2(\pi||p(\cdot|x))=I$, which is assumed to be finite and hence well defined.
    
    In the large $N$ limit it is clear that $\frac{(N-1)Z^2}{(w_i+S)^2}\approx \frac{1}{N}$. Thus, this expression can be slightly cleaned up if discarding higher order terms in $N$. More precisely, note that:
    \begin{equation}
        \frac{(N-1)Z^2}{(w_i+S)^2} = \frac{(N-1)Z^2}{(w_i+(N-1)Z)^2} = \frac{1}{N}+O(N^{-2}),
    \end{equation}
    as $w_i=O(1)$.
    Thus, equation \ref{eqn:exp_second_order} can be simplified to:
    \begin{equation}
    \label{eqn:exp_sec_order_simple}
        \mathbb{E}_{-i}\left[\left(\frac{\Delta}{w_i+S}\right)^2\right]=\frac{I}{N}+O(N^{-2}).
    \end{equation}
    
    Now all that remains is to bound the third order remainder term in equation $\ref{eqn:delta_taylor_exp}$ and aggregate results. Specifically, we wish to bound:
    \begin{equation}
        \left|\mathbb{E}_{-i}\left[\frac{\left(\frac{\Delta}{w_i+S}\right)^3}{1+\frac{\Delta}{w_i+S}}\right]\right|
        <\mathbb{E}_{-i}\left[\left|\frac{\left(\frac{\Delta}{w_i+S}\right)^3}{1+\frac{\Delta}{w_i+S}}\right|\right]
        =\mathbb{E}_{-i}\left[\left|\frac{\Delta^3}{(w_i+S)^2(w_i+S+\Delta)}\right|\right]
        <\mathbb{E}_{-i}\left[\left|\frac{\Delta^3}{S^2(w_i+S+\Delta)}\right|\right].
    \end{equation}
    This last expression is nontrivial to bound as $w_i+S+\Delta$, while necessarily positive, could be close to zero. To bound it we split the expectation into cases, where either $\Delta\geq-S/2$ or $\Delta<-S/2$. The first case is ``good'' in that the denominator is clearly nonzero, so we shall denote it with the indicator function $\mathds{1}_G$. Similarly, the second case is``bad'' in that the denominator could be close to $0$, so we will denote it with $\mathds{1}_B$, where clearly $\mathds{1}_G+\mathds{1}_B=1$ as the two cases are complementary. In the good case we have that:
    \begin{equation}
        \left|\frac{\Delta^3}{S^2(w_i+S+\Delta)}\right|\leq2\left|\frac{\Delta}{S}\right|^3.
    \end{equation}
    Now, we wish to bound $\left|\frac{\Delta}{S}\right|^3$. As the $w_j$ are IID and we assume finite third absolute moments ($\mathbb{E}_{-i}[|w_j-Z|^3]<\infty$), for some constant $C$ Rosenthal's inequality \citep{rosenthal1970subspaces} gives:
    \begin{equation}
        \mathbb{E}_{-i}[|\Delta|^3]\leq C\left(\sum_{j\not=i}\mathbb{E}_{-i}[|w_j-Z|^3]+\left(\sum_{j\not= i}\mathbb{E}_{-i}[(w_j-Z)^2]\right)^{3/2}\right),
    \end{equation}
    where note that each sum has $O(N)$ terms. Thus, ignoring constant factors:
    \begin{equation}
        \mathbb{E}_{-i}[|\Delta|^3]\leq C\left((N-1)O(1)+\left((N-1)O(1)\right)^{3/2}\right)=C\left(O(N)+\left(O(N)\right)^{3/2}\right)=O(N^{3/2}),
    \end{equation}
    which as $S=(N-1)Z$ implies $\mathbb{E}_{-i}[|\Delta/S|^3]=O(N^{-3/2})$ for the part of the expectation where $\Delta\geq-S/2$.

    As for the bad region, we shall show that the probability of $\Delta<-S/2$ is very small, so the contribution from the bad region is negligible. 
    
    Explicitly, note that $-S/2>\Delta=\sum_{j\not=i}w_j-S>-S$ as $w_j>0$. Thus, $|\Delta|<S$ and we have that:
    \begin{equation}
        \left|\frac{\Delta^3}{S^2(w_i+S+\Delta)}\right|<\frac{S^3}{S^2w_i}=\frac{S}{w_i}=O(N),
    \end{equation}
    as $w_i$ is constant and $S=(N-1)Z=O(N)$. Thus, we have that
    \begin{equation}
        \mathbb{E}_{-i}\left[\mathds{1}_B\left|\frac{\Delta^3}{S^2(w_i+S+\Delta)}\right|\right]
        <O(N)\mathbb{E}_{-i}[\mathds{1}_B]
        =O(N)\text{Pr}(\Delta<-S/2)
    \end{equation}
    Thus, we wish to find a sufficiently tight bound for $\text{Pr}(\Delta<-S/2)$.
    
    Note that $\Delta<-S/2 \implies |\Delta| > S/2$. Intuitively, for large $N$ the probability of the sum of displacements ($\Delta$) being larger than its mean by a constant factor (in this case $\frac{1}{2}$) is extremely low. Practically, we bound $\text{Pr}(\Delta<-S/2)$ using Hoeffding's inequality.

    Hoeffding's inequality requires bounded random variables, but $w_j$ is generically unbounded. Thus, for some $M>0$ define $y_j\coloneqq \min(w_j,M)$. As $w_j>0$ we have that $0<y_j\leq M$ and is thus bounded. As by definition $y_j\leq w_j$ we have that $\Delta+S=\sum_{j\not=i}w_j<S/2\Rightarrow \sum_{j\not=i}y_j<S/2$. Moreover, note that $\text{Pr}(\sum_{j\not=i}w_j<S/2)\leq\text{Pr}(\sum_{j\not=i}y_j<S/2)$, as if the first inequality holds ($\sum_{j\not=i}w_j<S/2$) the second inequality ($\sum_{j\not=i}y_j<S/2$) holds as well. Recall now that $M$ is arbitrary, thus we can define $M$ to be sufficiently large such that $\mathbb{E}_{-i}[y_j]\geq \frac{3}{4}Z$. The $\frac{3}{4}$ is arbitrary; any fraction larger than $\frac{1}{2}$ would suffice, as Hoeffding's inequality bounds the probability that a random variable would deviate from its mean by a constant factor. Thus, applying Hoeffding's inequality on the sum of $y_j$ implies that:
    \begin{equation}
        \text{Pr}(\Delta<-S/2)
        =\text{Pr}\left(\sum_{j\not=i}w_j<S/2\right)
        \leq \text{Pr}\left(\sum_{j\not=i}y_j<S/2\right)\leq e^{-cN}
    \end{equation}
    for some $c>0$. Thus, we have that:
    \begin{equation}
        \mathbb{E}_{-i}\left[\mathds{1}_B\left|\frac{\Delta^3}{S^2(w_i+S+\Delta)}\right|\right]<O(N)\text{Pr}(\Delta<-S/2)\leq O(N)e^{-cN}=o(N^{-3/2}).
    \end{equation}
    Technically this expression equals $o(N^{-k})$ for any positive $k$, but $3/2$ suffices here.

    Altogether, the results for the good and the bad regions imply that:
    \begin{equation}
    \label{eqn:third_term_bound}
        \begin{aligned}
            \left|\mathbb{E}_{-i}\left[
            \frac{\left(\frac{\Delta}{w_i+S}\right)^3}
            {1+\frac{\Delta}{w_i+S}}
            \right]\right|
            &< \mathbb{E}_{-i}\left[
            \left|\frac{\Delta^3}{S^2(w_i+S+\Delta)}\right|
            \right] \\
            &< 2\mathbb{E}_{-i}\left[
            \mathds{1}_G|\Delta/S|^3
            \right]
            + \mathbb{E}_{-i}\left[
            \mathds{1}_B
            \left|\frac{\Delta^3}{S^2(w_i+S+\Delta)}\right|
            \right] \\
            &< 2\mathbb{E}_{-i}\left[|\Delta/S|^3\right]
            + o(N^{-3/2}) \\
            &= O(N^{-3/2}) + o(N^{-3/2}) \\
            &= O(N^{-3/2}).
        \end{aligned}
    \end{equation}
    thus bounding the remainder term.

    Taking the results of equations \ref{eqn:exp_sec_order_simple} and \ref{eqn:third_term_bound}, plugging them into equation \ref{eqn:delta_taylor_exp} and taking an expectation gives:
    \begin{equation}
    \label{eqn:exp_delta_taylor}
         \mathbb{E}_{-i}\left[\left(1+\frac{\Delta}{w_i+S}\right)^{-1}\right]=1+\frac{I}{N}+O(N^{-3/2}).
     \end{equation}

     Plugging results back in now until we reach the original expression for $q(x'|x)$, plugging equation \ref{eqn:exp_delta_taylor} into equation \ref{eqn:exp_pre_expand} gives:
     \begin{equation}
         \mathbb{E}_{-i}\left[\frac{1}{w_i+\sum_{j=1,j\not=i}^N w_j} \right]=\frac{1}{w_i+S}\left(1+\frac{I}{N}+O(N^{-3/2})\right).
     \end{equation}
     Now, note that:
     \begin{equation}
         \frac{1}{w_i+S}=\frac{1}{NZ}\frac{1}{1+\frac{w_i-Z}{NZ}}=\frac{1}{NZ}\frac{1}{1-\frac{1-w_i/Z}{N}}=\frac{1}{NZ}\left(1+\frac{1-w_i/Z}{N}+O(N^{-2})\right)
     \end{equation}
     Thus, one has that:     
     \begin{equation}
        \label{eqn:clean_exp_invsum}
        \begin{aligned}
        \mathbb{E}_{-i}
        \left[
            \frac{1}{w_i+\sum_{j=1,j\not=i}^N w_j}
        \right]
            &= \frac{1}{NZ}
               \left(1+O(N^{-3/2})+\frac{I}{N}+\frac{1-w_i/Z}{N}+O(N^{-2})\right) \\
            &= \frac{1}{NZ}
               \left(1+\frac{I+1-\frac{\pi(x')}{p(x'|x)}}{N}+O(N^{-3/2})\right),
        \end{aligned}
    \end{equation}
    where the $O(N^{-2})$ term is absorbed by the $O(N^{-3/2})$. Plugging \ref{eqn:clean_exp_invsum} back into equation \ref{eqn:messy_trans_kernel} gives:
    \begin{equation}
        \begin{aligned}
        q(x'|x)
            &= NR(x')^\beta
            \mathbb{E}_{-i}
            \left[
                \frac{1}{w_i+\sum_{\substack{j=1\\j\ne i}}^N w_j}
            \right] \\
            &= \pi(x')
            \left(1+\frac{I+1-\frac{\pi(x')}{p(x'|x)}}{N}\right)
            + O(N^{-3/2}) \\
            &= \pi(x')
            \left(
                1+\frac{1}{N}\left(\int \frac{\pi(y)^2}{p(y|x)}dy-\frac{\pi(x')}{p(x'|x)}\right)
            \right)
            + O(N^{-3/2}).
        \end{aligned}
    \end{equation}
    Defining $D(x'|x)\coloneqq \int\frac{\pi(y)^2}{p(y|x)}dy-\frac{\pi(x')}{p(x'|x)}$ thus finishes the proof.
\end{proof}

The proof of corollary 2 is relatively straightforward. As the version stated in the main body is precise, we give only the proof here.
\begin{proof}[Proof of Corollary \ref{corr:chisq}, First order correction as a $\chi^2$ divergence]
    Directly calculating:
    \begin{equation}
        \begin{aligned}
        \mathbb{E}_{x'\sim p(\cdot|x)}[D(x'|x)]
            &= \int \frac{\pi(y)^2}{p(y|x)}
               - \frac{\pi(y)}{p(y|x)}p(y|x)\,dy \\
            &= \int \frac{\pi(y)^2}{p(y|x)}\,dy - 1 \\
            &= \int \frac{\pi(y)^2}{p(y|x)} - p(y|x)\,dy \\
            &= \int \left(
               \frac{\pi(y)-p(y|x)}{p(y|x)}
               \right)^2 p(y|x)\,dy.
        \end{aligned}
    \end{equation}
    where we used that $\int \pi(y)dy=\int p(y|x)dy=1$. This equalling $\chi^2(\pi||p(\cdot|x))$ is by definition of the $\chi^2$ divergence.
\end{proof}

\section{Regex Details}
\label{app:regex}
\subsection{Example Prompts}
Each regex in the dataset of \citet{siddiq2024re} has a prompt which briefly explains what the regex does and gives some examples for matches and non-matches. Table~\ref{tab:regex_prompts} contains these prompts and the dataset's reference regexes for the three regexes used as a validation set. Note that the prompts only show a few of the examples, while the accuracy is measured over all 25 (13 matches and 12 non-matches). All methods which get as input some prior regex as a hint used the prompt
\begin{quote}
\ttfamily Please output a regular expression that detects the following, with nothing preceding or succeeding it. \{description\}\\[0.5em]
The last proposed regex was \{previous regex\}. The regex is:
\end{quote}
where \texttt{\{description\}} is the regex' prompt from the dataset and \texttt{\{previous regex\}} is the hint, which is empty at the first step. Methods with no hint omit the sentence with \texttt{The last proposed regex was}.

\begin{table}[h]
\centering
\small
\begin{tabular}{p{0.25\linewidth} p{0.65\linewidth}}
Example correct regex & Description \\
\midrule
\texttt{\textasciicircum\textbackslash d+([\textasciicircum.,])?\$} &
Accepts positive integers.\newline
Match examples:\newline
- ``34''\newline
- ``1''\newline
- ``1000000000''\newline
Non-match examples:\newline
- ``-34''\newline
- ``-1'' \\
\midrule
\texttt{\textasciicircum\textbackslash d+\textbackslash/?\textbackslash d*\$} &
Used to validate fractions or UK odds.\newline
It will match integers, allowing 10/1 to be entered as 10, and fractions.\newline
It will not match negative numbers or fractions, spaces or any alphabetical characters.\newline
Match examples:\newline
- ``10''\newline
- ``2/3''\newline
- ``100/30''\newline
Non-match examples:\newline
- ``-10''\newline
- ``-2/3'' \\
\midrule
\texttt{\textasciicircum-?\textbackslash d+(\textbackslash.\textbackslash d+)?\$} &
Matches any unsigned or signed floating point number/numeric string.\newline
Match examples:\newline
- ``123''\newline
- ``3.14159''\newline
- ``-1.2''\newline
Non-match examples:\newline
- ``abc''\newline
- ``3.4.2'' \\
\end{tabular}
\caption{Prompts (\texttt{refined\_prompt} in the dataset) and reference regexes for the three validation regexes. References are the ground-truths given in the dataset of \citet{siddiq2024re}.}
\label{tab:regex_prompts}
\end{table}

\subsection{Regex Baselines}
\label{app:baselines_details}
In this section we review the different baselines used for the regex discovery task and how their setups and/or hyperparameters were tuned. Baselines not mentioned here do not have any hyperparameters and use the same prompts as DME. For all methods we used 3 regexes not included in the test set as validation regexes and used the hyperparameters that had the highest average success rate over these 3.

For \textbf{RS-GRPO} \citep{jiang2025risk} we sweeped their $\beta$ parameter and the learning rate. Note that the RS-GRPO objective is of the form $\frac{1}{\beta}\log\mathbb{E}[e^{\beta R(x)}]$, so their $\beta$ is different to the one used by other methods here. We ran a hyperparameter sweep over $\beta=2,8$, as these are the two $\beta$s used for the main results in \citet{jiang2025risk}, and learning rate$=10^{-4},10^{-5},10^{-6}$, finding that $\beta=8,\text{lr}=10^{-6}$ is best.

For \textbf{TTT-Discover} \citep{yuksekgonul2026learning} we initially tried to use an open-source replication that does not require the Tinker API \citep{tml2026tinker} but did not manage to get it to perform well on the regexes. Thus, we used the official TTT-Discover codebase in conjunction with skyRL tx \citep{griggs2025skrylv01} to simulate the Tinker backend. We ran a hyperparameter sweep over learning rates $10^{-5},10^{-6},10^{-7}$ and KL coefficients $0,0.1,0.01$, finding that $10^{-7}$ and $0$ respectively performed best.

For \textbf{GFlowNets} we implemented versions using both the trajectory balance loss \citep{malkin2022trajectory} and the subtrajectory balance loss \citep{madan2023learning}, finding that the latter gave slightly better performance. The subtrajectory balance loss slightly increases runtime relative to the trajectory balance loss as the former requires evaluating the reward given to all partial sequences and not only at the terminal state, which in our case is the final completion. We ran a hyperparameter sweep over $\beta=0,1,3,10,30,100,300,1000,3000$ and learning rates $10^{-6},10^{-7},10^{-8}$, finding that $1000,10^{-7}$ respectively performed best, although many $\beta$s and learning rates performed equally well.

For \textbf{GEPA} \citep{agrawal2026gepa} we manually iterated over the prompt and additional textual information given to the model. Moreover, we force the model to stop thinking and give an answer after $4096$ tokens as otherwise the model tends to exhaust its context length without outputting a new prompt. Otherwise, GEPA had no hyperparameters tuned.

\textbf{SCS} \citep{gideoni2026simple} used the same prompt as GEPA, only a single trial, and otherwise uses the same hyperparameters as given by \citet{gideoni2026simple}.

\subsection{Ablations}
In this section we test how different design choices, namely DME's beta parameter and the two tweaks to the base algorithm, affect its performance on the regex task. All experiments and reported performance are over the same 20 test regexes.

To test how DME's performance is affected by beta we sweep over betas from $10^0$ to $10^6$. Figure \ref{fig:dme_regex_beta_sweep} shows that above about $\beta=10^3$ DME on average performs the same. Note that this does not mean that DME's performance monotonically increases for all problems, e.g. for the third validation regex in Table \ref{tab:regex_prompts} $\beta=100$ has a success rate of 96\% while the success rate of $\beta=3000$ is 86\%.

\begin{figure}[h]
    \centering
    \includegraphics[width=0.7\linewidth]{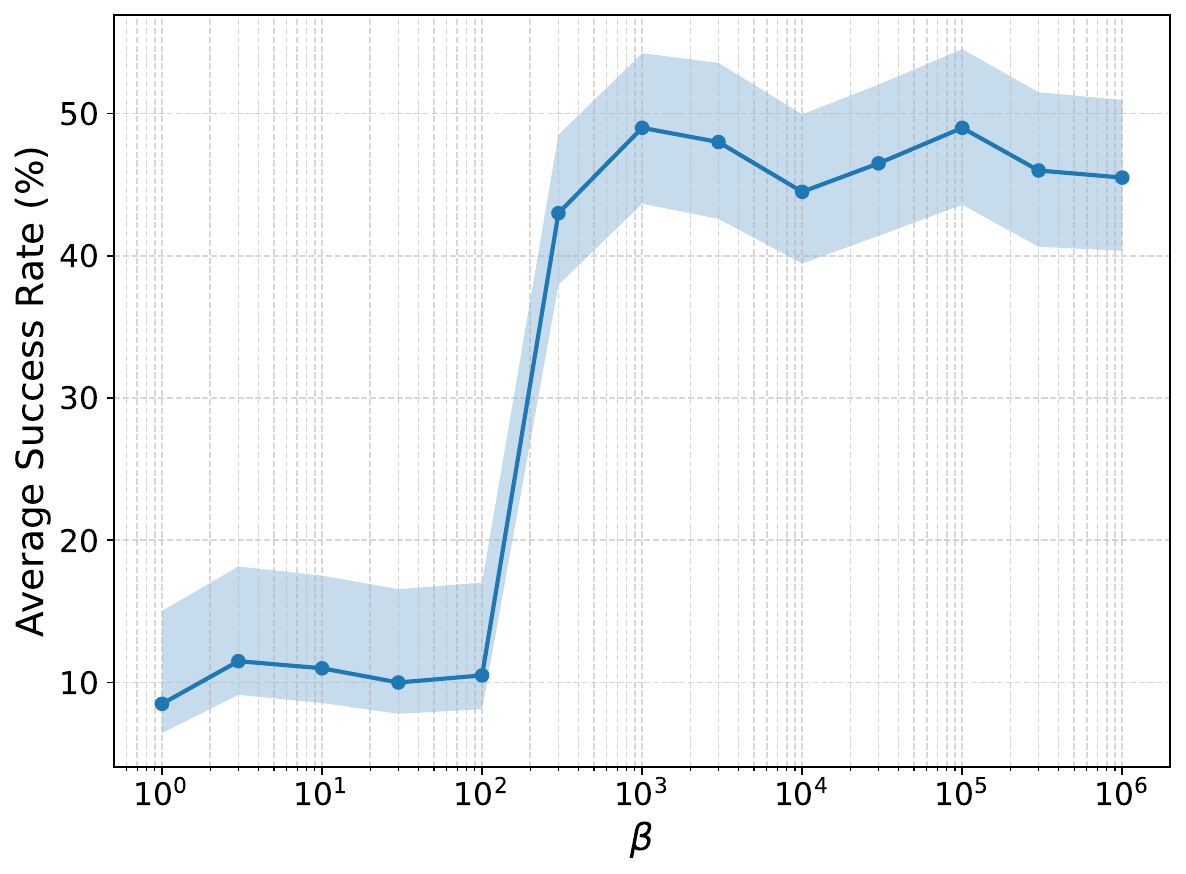}
    \caption{DME's average performance for the regex task increases with $\beta$, stabilizing after $\beta=10^3$. Uncertainties are 95\% confidence intervals.}
    \label{fig:dme_regex_beta_sweep}
    \vspace{-5pt}
\end{figure}

For the regex task it does not seem to matter much whether the path marginalization or length-respecting prior are used, see Table \ref{tab:regex_tweaks_ablate}. This makes sense as the output is relatively short and does not contain long, explicit reasoning leading up to it.

\begin{table}[h]
\centering
\begin{tabular}{l l|c c}
 & & \multicolumn{2}{c}{Path Marginalisation} \\
 & & \cmark & \xmark \\
\midrule
\multirow{2}{*}{\makecell[c]{Length-Respecting\\Prior}}
    & \cmark & 45.5\% & 49.5\% \\
    & \xmark & 46.5\% & 48.5\% \\
\end{tabular}
\caption{DME average success rate ($\beta=1000$), with and without the length-respecting prior and path marginalisation. 95\% confidence intervals are about $\pm5\%$ for all cells.}
\label{tab:regex_tweaks_ablate}
\end{table}

\section{Coding Competitions Details}
\label{app:code_comp}
\subsection{DME Ablations and Hyperparameter Selection}
To initially select which hyperparameters to use for DME we ran it on a random subset of 100 LiveCodeBench problems, distinct and harder than those used for comparing methods in Figure \ref{fig:lcb}. We sweep over $\beta=10^1,10^3,10^6$ and test DME with vs without the length-respecting uniform prior. Figure \ref{fig:code_hyp_sweep} shows the result of that sweep, where initially promising runs were run with an additional seed. $\beta=10^3$ seems to work better than $\beta=10^1$ or $\beta=10^6$, albeit only when using the length-respecting uniform prior.

\begin{figure}[!h]
    \centering
    \includegraphics[width=0.8\linewidth]{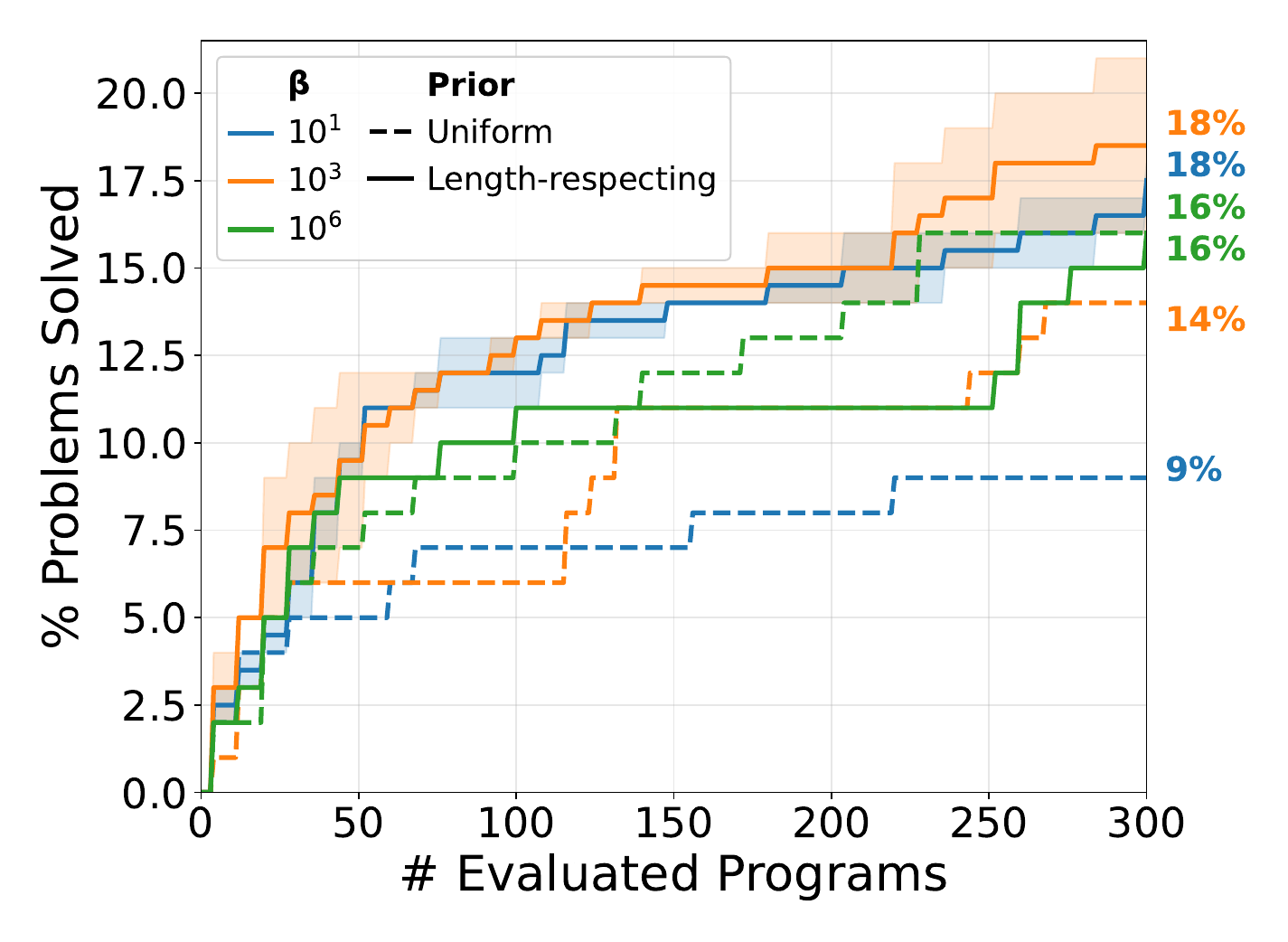}
    \vspace{-20pt}
    \caption{DME for different betas and with different priors, over a hard 100 question subset of LiveCodeBench. Configurations with shaded regions report mean and standard deviation over two seeds.}
    \label{fig:code_hyp_sweep}
\end{figure}

\subsection{Results for a Reasoning Model}
We compare DME to $(1,N)$ and $(1+N)$ when using a reasoning model, namely Qwen3-1.7B \citep{yang2025qwen3}. DME here uses $\beta=10^6$. Surprisingly, in this case we find DME to slightly underperform both evolutionary strategies. As the results for section \ref{sec:math_bounds} show, DME can be competitive when applied to reasoning models over many different problems, so it is unclear why DME underperformed here.

\begin{figure}[h]
    \centering
    \includegraphics[width=0.6\linewidth]{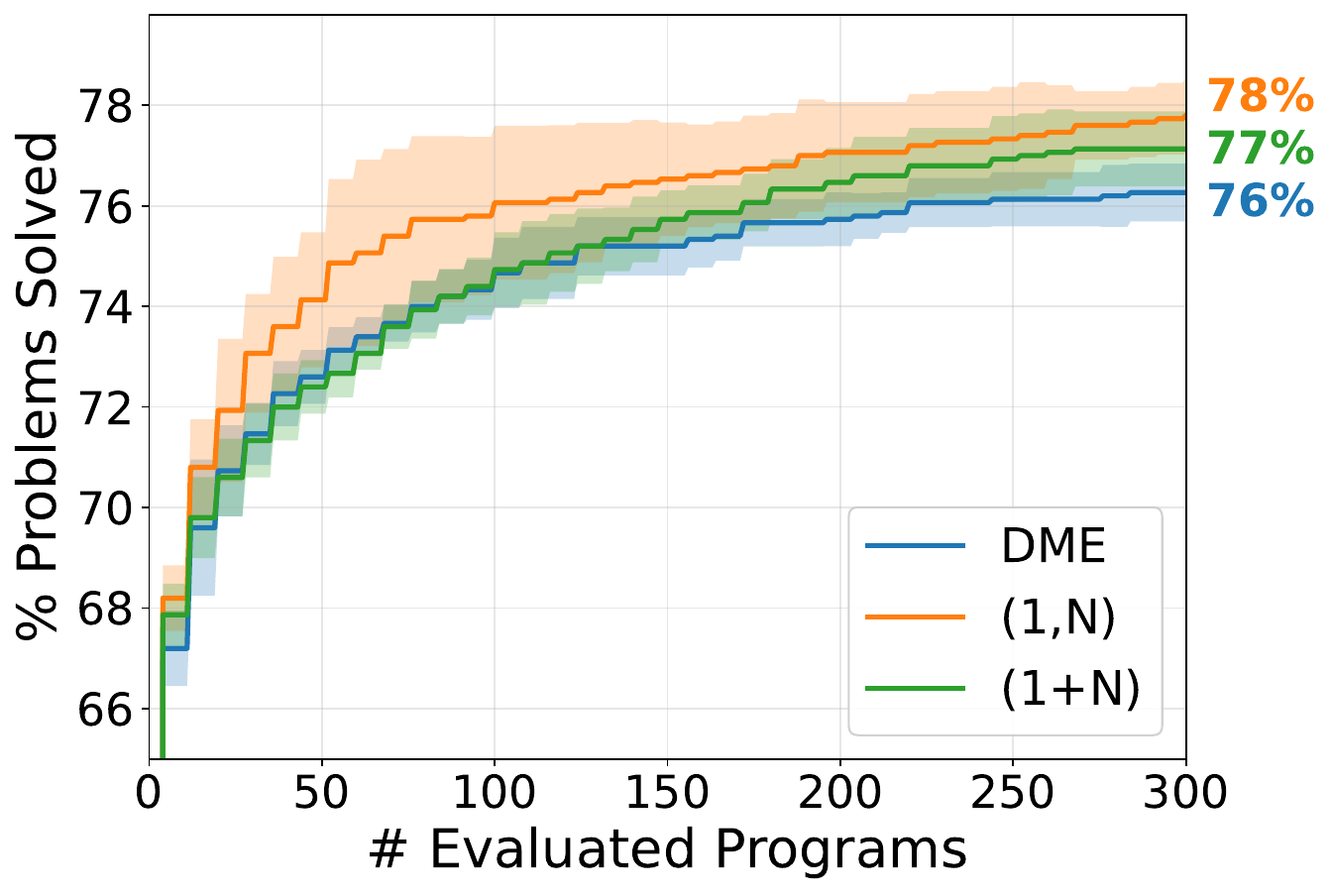}
    \caption{Different methods' performance when using a reasoning model as the proposer. Surprisingly, here DME slightly underperforms both evolutionary strategies.}
    \label{fig:dme_reasoning_neg}
\end{figure}

\section{Math Bounds Details}
\label{app:math_details}
\subsection{DME Details and Hyperparameter Tuning}
DME's reward is $e^{\text{bound}}$, with the bound's sign flipped when it needs to be minimized rather than maximized. Note that both BoN and SCS are invariant to the reward parameterization as they have no explicit fitness-based selection during the search, only selecting the min/max discovered bound after evaluating all programs.

DME's hyperparameters were tuned by running short, single-chain runs that generate up to 1000 programs on the third autocorrelation inequality problem from the original AlphaEvolve paper \citep{novikov2025alphaevolve}. This problem was not studied in \citet{gideoni2026simple}, hence it being excluded from the main results. Table \ref{tab:dme_tuning} shows the results found per run. Both using the standard prior and increasing beta seemed to help. 

\begin{table}[h]
\centering
\begin{tabular}{@{}c cc@{}}
& \multicolumn{2}{c}{Length-respecting prior} \\
\cmidrule(lr){2-3}
$\beta$ & \cmark & \xmark \\
\midrule
$10^{3}$ & 1.556 & 1.508 \\
$10^{6}$ & 1.489 & \textbf{1.488} \\
\end{tabular}
\caption{DME hyperparameter tuning on the third autocorrelation inequality problem, lower is better. Each entry is the best discovered upper bound.}
\label{tab:dme_tuning}
\end{table}

\subsection{Probability of Improvement Calculations}
Additional programs were generated to get a better estimate for the probability of improvement, with the average then being calculated by bootstrapping. Specifically, we sample an additional 1000 programs from both BoN and SCS. More samples are generated from SCS by generating more trials. We sample 3 additional chains for DME, each running for the same number of steps as the original 4 did. The bootstrapping consists of sampling from each method without replacement, namely from the BoN programs, the SCS trials, and the DME chains.

\subsection{Per-Problem Breakdown}
\label{app:math_per_problem}
See Figure \ref{fig:per_prob_prob_imp}. Note that per-problem estimates are much less reliable than the aggregate estimate due to having less data.

\begin{figure}
    \centering
    \includegraphics[width=\linewidth]{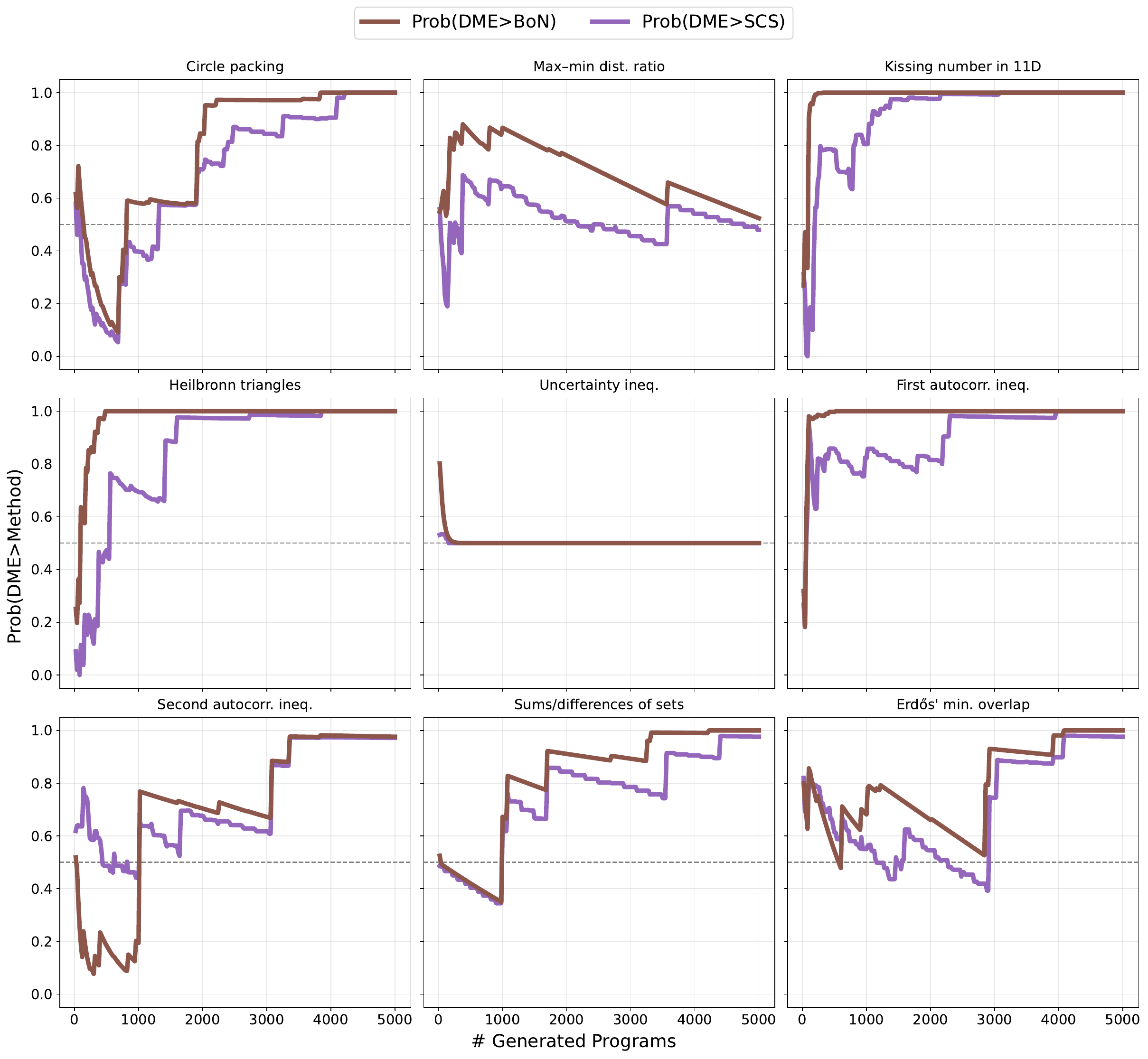}
    \caption{Per-problem DME probability of improvement.}
    \label{fig:per_prob_prob_imp}
\end{figure}

\end{document}